\documentclass[12pt]{colt2018}

\title[A Direct Sum Result for the Information Complexity of Learning]{A Direct Sum Result for the \\ Information Complexity of Learning}
\usepackage{times}
\usepackage{soul}
\usepackage{enumitem}

\usepackage{graphicx}
\usepackage{hyperref}
\usepackage[hang,flushmargin]{footmisc}
\usepackage{scrextend}

\usepackage{mathptmx}

\usepackage{pgfplots}

 \coltauthor{
 \Name{Ido Nachum}
 \Email{idon@tx.technion.ac.il}\\
 \addr{Department of Mathematics, Technion-IIT, Haifa, Israel.}
 \AND
 \Name{Jonathan Shafer} \Email{shaferjo@berkeley.edu}\\
 \addr{Computer Science Division, University of California, Berkeley, CA.}
 \AND
 \Name{Amir Yehudayoff} \Email{amir.yehudayoff@gmail.com}\\
 \addr{Department of Mathematics, Technion-IIT, Haifa, Israel.}
 }
\DeclareMathAlphabet{\mathcal}{OMS}{cmsy}{m}{n}

\newcommand{\HH}{\mathrm{H}}
\newcommand{\E}{\mathbb{E}}
\newcommand{\EE}[1]{\underset{#1}{\mathbb{E}}}
\newcommand{\mc}[1]{\mathcal{#1}}
\newcommand{\mb}{\mathbb}
\newcommand{\mr}{\mathrm}
\newcommand{\cH}{\mc{H}}
\newcommand{\cX}{\mc{X}}
\newcommand{\cY}{\mc{Y}}
\newcommand{\cA}{\mc{A}}
\newcommand{\cT}{\mc{T}}
\newcommand{\cS}{\mc{S}}
\newcommand{\II}[1]{\mathrm{I}\left(#1\right)}
\newcommand{\I}[2]{\underset{#1}{\mathrm{I}}\left(#2\right)}
\newcommand{\supp}{\mathrm{supp}}
\newcommand{\IC}{\mathsf{IC}}
\newcommand{\VC}{\mathsf{VC}}
\newcommand{\nin}{\noindent}

\newtheorem{notation}{Notation}[section]
\newtheorem{claim}{Claim}[section]

\usepgfplotslibrary{polar}
\usepgflibrary{shapes.geometric}
\usetikzlibrary{calc}
\pgfplotsset{compat=1.12}
\usepgfplotslibrary{fillbetween}
\usetikzlibrary{patterns}

\begin{document}

\maketitle

\newcommand{\an}[1]{\textcolor{blue}{#1}}
\newcommand{\ann}[1]{\textcolor{red}{#1}}

\begin{abstract}
How many bits of information are required to PAC learn a class of hypotheses of VC dimension $d$? 
The mathematical setting we follow  is that of Bassily et al., where the value of interest is the mutual information $\II{S;A(S)}$ between the input sample $S$ and the hypothesis outputted by the learning algorithm $A$. We introduce a class of functions of VC dimension $d$ over the domain $\cX$
with information complexity at least $\Omega \left(d\log \log \frac{|\cX|}{d}\right)$ bits for any consistent and proper algorithm (deterministic or random).  
{Bassily et al.\ proved a similar (but quantitatively weaker) result for the case $d=1$.}

The above result is in fact a special case of a more general phenomenon we explore. We define the notion of {\em information complexity} of a given
class of functions $\cH$. Intuitively, it is the minimum amount of information that an algorithm for $\cH$ must retain about its input to ensure consistency and properness. We prove a direct sum result for information complexity in this context; roughly speaking, the information complexity sums when combining several classes.\\
\end{abstract}


\section{Introduction}

A simple, interesting and central observation in computational learning theory suggests that {\em learning} and {\em information compression} are closely related tasks, and in some sense are equivalent: In order to compress a dataset, one needs to identify patterns or regularities that exist in the data, and leverage them to construct a representation that is more concise than the verbatim description of the data. Similarly, learning involves identifying patterns or regularities in the training data, usually for the purpose of making predictions about some future data that is expected to exhibit similar patterns.

This intuition has been formalized a number of times in various ways;
for example, via sample compression schemes, Occam's razor, and minimum description length
(see Section~\ref{related_work}). The observation is fruitful not least because it enables borrowing tools from the study of compression, which are often combinatorial or information-theoretic in nature, and applying them to study statistical notions of learning. This paper presents a result pertaining to the {\em information complexity} of hypothesis classes, which is yet another formalization. 

The information complexity view roughly goes as follows. Consider a learner that is given some training data, and will later need to make predictions about new data that it has not yet seen. Once the learner is successful in identifying the underlying patterns in the data, we expect it will be able to keep just a small amount of information which represents these patterns, discard the rest of the training data, and still be successful in making predictions about future instances. Thus, it is natural to ask how much information the learner retains from the training set.

This question was first introduced in \cite{bassily2018learners}, where it is formalized as follows (see Section~\ref{section-preliminaries} for definitions and notation). Let 
$S$ be a random variable representing the i.i.d.\ training samples for a supervised learning algorithm 
and let $h$ 
represent the hypothesis that the algorithm outputs on input $S$. Consider the mutual information $\II{S;h}$, which quantifies the amount of information that the algorithm retains about the input when making predictions. The authors show that
\[
\Pr \left[ \left|\text{true error} - \text{empirical error}\right| > \varepsilon \right] 
< O\left(\frac{\II{S;h}}{m\varepsilon^2}\right) .
\]
where $m$ is the number of samples in the training set. Thus, if the mutual information grows slowly compared to $m$, say $\II{S;h} = o(m)$, then the algorithm generalizes well, meaning that it does not over-fit. If in addition the empirical error vanishes then the algorithm PAC learns. Similar results were also obtained by \cite{xu2017information}.

Again, this makes sense. If the mutual information does not grow a lot slower than $m$, then the algorithm is basically memorizing large portions of the training set, and therefore its output is unlikely to generalize to unseen instances.  Conversely, if the mutual information is small compared to $m$ then the algorithm cannot tailor its output to the noisy details of specific training instances, and so it cannot over-fit.

More generally, given a hypothesis class $\cH$, we define the information complexity $\IC(\cH)$ to be the least amount of information that the output of an algorithm for $\cH$ must retain to ensure consistency and properness. 
{Here is a brief and rough overview of the formalism behind this notions
(see Definition~\ref{definition-IC} for the full details).
The {\em information cost} of a consistent and proper learning algorithm $A$ for $\cH$
is $\IC_{A}(\cH) = \sup_p \II{S;A(S)}$ where $p$ is a distribution on inputs and $\mathrm{I}$ denotes mutual information (here the sample size $m$ is fixed). 
In words, it is the maximum amount of information the output of $A$ contains
over all distributions on inputs.
The {\em information complexity} of $\cH$
is $\IC(\cH) = \inf_A \IC_A(\cH)$.
}

This definition is the main object of study in this work. As the name suggests, 
the definition is inspired by {similar notions in computational
complexity theory, like} the concept of information complexity in the field of communication complexity \citep[see e.g.][]{braverman2012interactive}.
{Loosely speaking, algorithms have costs and the minimum cost 
for a given problem is the complexity of the problem.}

{The information complexity of a learning problem seems related
to important properties of the problem, like the sample size needed to perform learning,
and may yield new and useful learning paradigms,
which aim at minimizing the amount of information algorithms use.
It is also related to several standard notions in learning theory
(see Section~\ref{related_work} below).}
We therefore set out to explore this measure and understand how it works. 

One question that arises immediately inquires how, if at all, does information complexity relate to the Vapnik--Chervonenkis (VC) dimension, which is the standard measure of complexity in learning theory. VC dimension is important because a hypothesis class is PAC-learnable if and only if its VC dimension is finite, and furthermore, when learning is possible then the VC dimension determines the sample complexity \citep{vapnik1971uniform, blumer1989learnability}. 

This work makes a first step towards understanding information complexity, and its relation to the VC dimension, by proving the following theorem.

\begin{theorem}[\textbf{Lower bound for VC classes}]\label{general-lower-bound} 
There exists a family of hypothesis classes $\{\cH_{k,d}:\: k,d,n \in \mathbb{N}, k=d2^n\}$ where $\cH_{k,d}\subseteq \{0,1\}^{[k]}$ and $\VC(\cH_{k,d})=d$
such that
\[
\IC(\cH_{k,d})=\Omega\left(d\log \log \left(k /d\right)\right)=\Omega\left(d\log n\right) .
\]
\end{theorem}

This theorem shows that in some cases learners must retain or leak a large amount of information about their inputs, even for classes of low VC dimension. The theorem applies to both deterministic and randomized algorithms. A weaker variant of it, with $d=1$, already appeared in \cite{bassily2018learners}. It also provides a separation between information complexity and sample compression schemes (as explained below).

The theorem is in fact an instance of a much more general direct-sum-type phenomenon. In a nutshell, direct sums in computational complexity theory refer
to the behavior of the computational complexity when problems are combined. For example, how does the complexity of completing two tasks relate to the complexity of completing each of the tasks separately. It is a central questions that appears in boolean formula complexity \citep[see e.g.][]{karchmer1995super}, communication complexity \citep[see e.g.\ Section 4.1 in][]{kushilevitz1997communication}, and more.

Here we describe the statement in rough terms (for formal details see Section~\ref{direct-sum}).
Given two concept classes $\cH_1 \subseteq \{0,1\}^{\cX_1}$
and $\cH_2 \subseteq \{0,1\}^{\cX_2}$, define the product class
$\cH_1 \times \cH_2$ as the class of functions over the disjoint union of $\cX_1$ and $\cX_2$
that are obtained by combining some $h_1 \in \cH_1$ and $h_2 \in \cH_2$.
The main question we address is how does the information complexity
of $\cH_1 \times \cH_2$ relate to that of $\cH_1$ and $\cH_2$.
It is natural to {conjecture} that the information complexity sums;
namely, $$\IC(\cH_1 \times \cH_2) \approx \IC(\cH_1) + \IC(\cH_2).$$
We prove that this is indeed the case.


\subsection*{Related work}\label{related_work}

Connections between learning and compression have been studied extensively. 
The minimum description length principle developed by Rissanen {and others} is one important avenue \citep{rissanen1978modeling, grunwald2007minimum}, as is Solomonoff induction, which relates to compression via Kolmogorov complexity (\citealt{solomonoff1964formalA, ming1997introduction}; see also \citealt{hutter2007universal}). Below we discuss two seminal results relating learning and compression which we find particularly pertinent.

\textbf{Sample compression schemes.} The concept of information complexity was first conceived as an attempt to generalize sample compression schemes, which constitute a well-known connection between learning and compression. Sample compression schemes are a class of learning algorithms whose output hypothesis is determined by a small subsample of the input. The classic example is support vector machines, which output a separating hyperplane that is determined by a small number of support vectors. \cite{littlestone1986relating} introduced sample compression schemes, and showed that every sample compression scheme is a PAC learner. Recently, \cite{moran2016sample} resolved a longstanding open question and showed that the converse also holds: every hypothesis class of finite VC dimension has a compression scheme. Together, these two results show one way in which learning and compression are equivalent. The present work extends \cite{bassily2018learners} in showing that learning with small mutual information and sample compression schemes are different.

\textbf{Occam's razor.} Another classic connection between learning and compression was provided by \cite{blumer1987occam}. They assume some fixed encoding of the hypotheses in a class $\cH$, and define the complexity of a hypothesis to be the length of its encoding. Roughly, they show that if an algorithm always outputs a consistent and relatively simple hypothesis then the algorithm 
generalizes. More explicitly, they show a bound on the sample complexity under the condition that the output hypothesis is of some given complexity. 
Since information complexity is a lower bound on the entropy of the output hypothesis, which itself is a lower bound on the length of its encoding, the lower bound for information complexity proved herein carries over also to the setting of Occam's razor.

\textbf{Differential privacy.} 
The mutual information $\II{S;A(S)}$ can also be thought of as the amount of information the learning algorithm reveals about its input, which calls to mind the setting of differential privacy. First introduced by \cite{dwork2006calibrating}, differential privacy is a principled notion of privacy that has recently been studied extensively because it provides strong privacy guarantees to data sources in a rigorous way. It also plays a roll in controlling overfitting, as several recent works  have shown \cite[e.g.][]{dwork2015preserving, bassily2016algorithmic, rogers2016max, bassily2014private}. Specifically, \cite{bassily2016algorithmic} provides an analysis of differential privacy as a form of distributional stability, and provides a tight characterization of the generalization guarantees that differential privacy entails.
The ideas presented in this work
may carry over to show direct sums in differential privacy.

\subsection*{Proof Outline}\label{Proof Sketch}

Here we provide an outline of the proof's structure.
The proof consists of four parts.

In section \ref{lower-thresholds}, we consider the class of thresholds over 
the domain $\cX$. We improve upon the lower bound in Theorem 5.1 of \cite{bassily2018learners}, and show that for every sample size $m$ and every consistent and proper learning algorithm, there exists a distribution and a threshold function such that $\II{S;A(S)}=\Omega ( \log \log | \cX | )$. This corresponds to the case $d=1$ in Theorem~\ref{general-lower-bound}.

In section \ref{direct-sum}, we define the direct sum of classes of functions $\cH_1,\ldots,\cH_d$. Every function in the new class is the concatenation of $d$ functions, one from each class. In particular, we are interested in the direct sum of $d$ classes of threshold functions.  The new class has VC dimension $d$ and is denoted $\cT_{k,d}$.

{Then,} 
we prove that (under certain conditions) the information complexity of the direct sum of $d$ classes is roughly the sum of their information complexities. This is the main technical contribution of this work. It harnesses {Sion's generalization of} von Neumann's minimax theorem in a somewhat surprising way; instead of considering the space of distributions over the domain of interest, we need to move to the space of distributions over distributions.
Finally, in section \ref{VC bound}, we verify that the relevant conditions hold for $\cT_{k,d}$ and conclude our proof. 

\section{Preliminaries}\label{section-preliminaries}

\subsubsection*{Standard notation}



\begin{notation}\label{definition-delta}
	Let $\cX$ be a set. The support of a probability function $p:\cX \to [0,1]$ is the subset of elements of $\cX$ that have a positive probability, $\supp(p)=\{x\in \cX:\: p(x)>0\}$. We use $\Delta(\cX)$ to denote the set of all probability mass functions over $\cX$ that have a finite support. If $p$ is a probability function, we use $p^m$ to denote the probability function over $\cX^m$ that corresponds to the probability of performing $m$ i.i.d.\ samples from $\cX$ according to $p$: $p^m((x_1,\cdots,x_m)) = \Pi_{i=1}^m p(x_i)$.
\end{notation}
%

\subsubsection*{Information theory}

\begin{definition}
	Let $\cX$ be a countable set, and let $X$ be a random variable over $\cX$ with probability mass function $p$ such that $p(x)=\Pr(X = x)$. The {entropy of $X$} is\footnote{$\log (x)$ is a shorthand for $\log_2(x)$, and we use the convention that $0\log\frac{1}{0}=0$.}
$\HH(X) = \sum_{x \in \cX} p(x)\log\frac{1}{p(x)}$.
\end{definition}


\begin{definition}
	Let $X$ and $Y$ be random variables over countable sets $\cX$ and $\cY$ respectively. The mutual information between $X$ and $Y$ is $\II{X;Y} = H(X)+H(Y)-H(X,Y)$.
\end{definition}

See the textbook \cite{cover2012elements} for additional basic definitions and results from information theory which are used throughout this paper.

\subsubsection*{Learning theory}

Part I of \cite{shalev2014understanding} provides an excellent comprehensive introduction to computational learning theory. Following are some basic definitions. 

\begin{definition}
	Let $\cX$ and $\cY$ be sets. $\cH$ is called {a class of hypotheses} if $\cH \subseteq \cY^\cX$.
%
%
	$\cS = \cX \times \cY$ is called the {sample space}.
%
	A {realizable sample for $\cH$ of size $m$} is 
	\[
	S = \big((x_1,y_1), \dots, (x_m,y_m) \big) \in \cS^m
	\]
	such that there exists $h \in \cH$ satisfying $y_i=h(x_i)$ for all $i \in [m]$.
\end{definition}

\begin{definition}\label{def_learning_algo}
	A {learning algorithm for $\cH$ with sample size $m$} is a (possibly randomized) algorithm that takes a realizable sample $S = ((x_1,y_1), \dots, (x_m,y_m))$ for $\cH$ as input, and returns a function $h:\cX\rightarrow \cY$ as output. We say that the learning algorithm is {consistent} if the output $h$ always satisfies $y_i=h(x_i)$ for all $i \in [m]$. We say the algorithm is { proper } if it outputs members of $\cH$.
\end{definition}


\begin{definition}
	Let $p \in \Delta(\cS)$. We say that {$p$ is realizable by $\cH$} or equivalently that {$p$ is consistent with $\cH$} if there exists $h\in\cH$ such that for all $(x,y)\in \supp(p)$ it holds that $y=h(x)$.
	We use $\Delta_{\cH}$ to denote the set of all distributions in $\Delta(\cS)$ that are consistent with $\cH$.
\end{definition}

%

\begin{definition}
	We say that $\cH \subseteq \{0,1\}^\cX$ shatters some finite set $C\subseteq \cX$ if\:\footnote{$\cH_C$ is the restriction of $\cH$ to $C$.}
	$\cH_C=\{0,1\}^C$.
	The {VC dimension of $\cH$} denoted $\VC(\cH)$ is the maximal size of a set $C \subseteq \cX$ such that $\cH$ shatters $C$. If $\cH$ can shatter sets of arbitrary size, we say that the VC dimension is $\infty$.
\end{definition}


\subsubsection*{Information complexity}

\begin{definition}
	Let $A$ be a learning algorithm for $\cH$ with sample size $m$, 
	and let $h$ be
	the output of $A$ when executed with input $S$. We say that $A$ retains at most $d$ bits of information from $S$
	if
	\[
	\forall p \in \Delta_\cH \ \  \I{S \sim p^m}{S;h} \leq d .
	\]
\end{definition}

\begin{definition}\label{definition-IC}
Let $\cA_m$ denote the set of (possibly randomized) consistent and proper learning algorithms for $\cH$ over samples of size $m$. {The information complexity of $\cH$ for samples of size $m$
is 
$$\IC_m(\cH) = \inf_{A \in \cA_m} \: \sup_{p \in \Delta_\cH} \I{S \sim p^m}{A(S);S}.$$}
The information complexity of $\cH$ is
$\IC(\cH) = \sup_{m \in \mb{N}} \IC_m(\cH)$.
\end{definition}

Conceptually,  the information complexity is the minimal amount of information that an algorithm must retain 
in order to be consistent and proper.

\section{The Lower Bound for Thresholds}\label{lower-thresholds}

{We start with the case $d=1$, and with the class of thresholds.}

\begin{definition}
	The {class of threshold functions of size $2^n$} is denoted $\cT_n$ and defined as follows:
	Let $\cX = [2^n]$ and $\cY=\{0,1\}$. The set $\cT_n\subseteq \cY^\cX$ consists of all monotone increasing functions; that is, $\cT_n = \{f_k\}_{k \in [2^n+1]}$ where
	\[
	f_k(x) = 
	\left\{
	\begin{array}{ll}
	0  & x < k \\
	1 & x \geq k
	\end{array}
	\right.
	\]
\end{definition}

\nin Following is our main result for thresholds.

\begin{theorem}[\textbf{Lower bound for thresholds}]\label{lower-bound-for-thresholds} 
For any (possibly randomized) proper and consistent learning algorithm $A$ for $\cT_n$ with sample size $m \geq 2$ there exists a distribution $p \in \Delta_{\cT_n}$ such that
	\[
	\I{S \sim p^m}{A(S);S} = \Omega(\log n) = \Omega(\log \log |\cX|) .
	\]
\end{theorem}

This theorem (which is proved in Appendix~\ref{sec:LBforT}) is an improvement upon Theorem 5.1 from~\cite{bassily2018learners}.
The proof of Theorem~\ref{lower-bound-for-thresholds} follows the same outline
as the proof of~\cite{bassily2018learners}.
Their proof, however, was based upon conditioning on an event of small probability $O(\frac{1}{m^2})$, so they arrived at a quantitatively weaker bound of $\Omega(\frac{\log n}{m^2})$. 

We are able to remove this dependence by using the following simple, yet useful {and general} observation: Permuting the ordered set of samples does not increase the mutual information.
{In other words, a learning algorithm that aims at minimizing
the information cost should not use the order in which the examples appeared.}

\begin{lemma}\label{permutation-decreases-information}
Let $\cH$ be a class of hypotheses and $A$ be a learning algorithm that accepts samples of size $m$. Define $A'$ to be $A'(S)=A(\Sigma(S))$ where $\Sigma$ is a random permutation chosen uniformly from all permutations on $m$ elements (independently of the input sample and the random coins of $A$). Then	for all $p\in\Delta_\cH$,
	\[
	\I{S \sim p^m}{A(S);S} \geq \I{S \sim p^m}{A'(S);S} .
	\]
\end{lemma}

\begin{proof}[Lemma~\ref{permutation-decreases-information}].
	\begin{align*}
	\HH(A(S) \: | \: S) &\stackrel{(a)}{=} \HH(A(\Sigma(S))\: | \:\Sigma(S)) 
	 \stackrel{(b)}{=} \HH(A(\Sigma(S))\: | \:\Sigma(S), S, \Sigma) 
	 \leq \HH(A(\Sigma(S))\: | \: S)
	\end{align*}
	where $(a)$ holds because $S$ and $\Sigma(S)$ have the same distribution, and $(b)$ holds because $A(\Sigma(S))$ and $(S, \Sigma)$ are independent conditioned on $\Sigma(S)$.
	Therefore,
	\begin{align*}
	\I{S \sim p^m}{A(S);S} & = \HH(A(S))-\HH(A(S) \: | \: S) \\
	 &\geq \HH(A(\Sigma(S))) - \HH(A(\Sigma(S))\: | \: S) \\
	& = \I{S \sim p^m}{A(\Sigma(S));S}
	\end{align*}
	as desired.
\end{proof}

\section{The Direct Sum}\label{direct-sum}

Our strategy for proving the lower bound in Theorem~\ref{general-lower-bound} is to use a reduction to the lower bound for thresholds. 
The main step in the proof is a direct sum result for information complexity
(as was discussed in the introduction).
We know that the information complexity of thresholds $\cT_n$ is high.
Our goal is to prove that this implies that the information complexity
of the class of functions that have $d$ thresholds is $d$ times as large
as that of $\cT_n$.

The statement we prove can be more generally stated using the following
terminology.

\begin{definition}
	Let 
	\[
	\cH_1 \subseteq \cY_1^{\cX_1},\dots,\cH_d \subseteq \cY_d^{\cX_d}
	\]
	be classes of hypotheses, where $\cX_1,\dots,\cX_d$ are disjoint sets. The $d$-fold product of $\cH_1,
		\dots,\cH_d$ is defined by:
	\[
	\cX = \bigcup_{i=1}^d \cX_i, \:\:\:\:\: \cY = \bigcup_{i=1}^d \cY_i ,
	\]
and
	\[
	\cH = \cH_1\times\cH_2\times\cdots\times\cH_d = \left\{(h_1,\dots,h_d):\:\forall i \ \ h_i\in\cH_i\right\} \subseteq \cY^{\cX}
	\]
	such that if $h=(h_1,\dots,h_d)$ and $x\in\cX$ then 
	$
	h(x)=h_i(x)
	$
	where $i$ is the integer for which $x \in \cX_i$.
\end{definition}


\nin {We are interested in the behavior of the VC dimension
with respect to the product of classes.}

\begin{lemma}\label{clm:VCdimofX}
$\VC(\cH_1\times\cH_2\times\cdots\times\cH_d) = \sum_{j=1}^d \VC(\cH_i)$.
\end{lemma}




We now state the direct sum result for information complexity.
Namely, that the information complexity of the product class roughly equals the sum of the information complexities. We do this using the two theorems below (all further proofs are deferred to the appendices).

The first theorem enable us to treat the mutual information on each summand separately. 

\begin{theorem}[\textbf{Direct sum I}]\label{superadditivity}
	Assume that 
	\begin{itemize}
		\item{$\cH$ is the product of $\cH_1,\dots,\cH_d$.}
		\item{$A$ is a (possibly randomized) learning algorithm for $\cH$.}
		\item{$p\in\Delta_\cH$ and $S \sim p^m$.}
		\item{$A(S)=(h_1,\dots,h_d)$.}
		\item{For each $i \in [d]$,
			\begin{itemize}
				\item{$S_i$ is the subsample of $S$ of examples from $\cX_i\times\cY_i$.}
				\item{$M_i = |S_i|$ is the number of examples in $S_i$.}
			\end{itemize}
		} 
	\end{itemize}
	Then
	\[
	\I{S \sim p^{m}}{A(S);S} \geq \sum_{i=1}^d \I{S \sim p^{m}}{h_i;S_i|M_i}
	\]
\end{theorem}

{Here is a simple illustration of the usefulness of the theorem.
Consider a learning algorithm $A$ for the product class $\cH$.
It gets as input a list of examples $S$,
and it knows to which $\cH_i$ each example corresponds to.
When applying the chain rule on $\II{A(S);S}$,
we see e.g.\ that when analyzing the information relevant to $\cH_2$
we need to condition on examples for $\cH_1$.
This in turn implies that we are not dealing with a single algorithm for $\cH_2$,
but with a family of algorithms that implicitly depend on $S$.
The theorem, however, shows that all of this can be ignored
without a significant price.
The only thing that matters is how many examples 
that correspond to $\cH_2$ there are.
}

{To complete the proof,} we now need to find a hard distribution for the product class, one that yields large information complexity. A natural idea would be to use the hard distributions of the individual classes. However, this approach fails, since there is no single hard distribution; it follows from \cite{bassily2018learners} for the case of thresholds that for every algorithm there is a hard distribution but for every distribution there is also an algorithm that retains little information {($O(1)$)}.

To solve this problem, we need to replace the notion of a hard distribution,
by the notion of a hard distribution on distributions.
Indeed, we show that for thresholds (and in fact more generally)
there is a single distribution on distributions that
for every algorithm yields high information cost (on average).
To this end, we use the spirit of von Neumann minimax theorem (Theorem~\ref{von-neumann-minimax}) and write the following (we actually need Sion's minimax theorem, Theorem~\ref{sion-minimax}):
\[
\min_{A}\max_{D\in\Delta(\Delta_{\cH})} f(A,D) = \max_{D\in\Delta(\Delta_{\cH})}\min_{A} f(A,D)
\]
where 
$
f(A,D) = \EE{p \sim D} \: \I{S \sim p^m}{A_m(S);S}
$.
This enables us to find a distribution over the space of distributions that is hard for all algorithms. 
This rational can be useful in other contexts where the minimax theorem doesn't hold;
although the minimax theorem does not apply for distributions versus algorithms in this context,
it does hold for distributions over distributions versus algorithms.

{
We still have one more technical difficulty to handle: Each hard distribution is tailored for a specific $m$, and we do not know $m$ in advance. To address this issue, we need to consider the setting where the number of samples $m$ is also random, as in the following lemma. 
}
\begin{definition}
	Let $M \subseteq \mb{N}$ and let $\cH \subseteq \cY^\cX$ be a class of hypotheses. A learning algorithm for $\cH$ that accepts samples of sizes in $M$ is a vector
$\left(A_m\right)_{m\in M}$ such that each $A_m$ is a learning algorithm for $\cH$ that accepts samples of size $m$. We say that $\left(A_m\right)_{m\in M}$ is consistent and proper if all the algorithms in the vector are consistent and proper.
\end{definition}

The combination of the minimax idea together with a randomized sample size is summarized in the following lemma.

\begin{lemma}\label{lemma-global-distribution-extended}
	Let $c\in\mathbb{R}$, $M \subseteq \mb{N}$, $|M|<\infty$, $\mu\in\Delta(M)$ and let $\cH \subseteq \cY^\cX$ be a class of hypotheses. Assume that for any consistent and proper learning algorithm $\left(A_m\right)_{m\in M}$ for $\cH$ there exists $p \in \Delta_{\cH}$ such that 
	\[
	\EE{m \sim \mu} \: \I{S \sim p^m}{A_m(S);S} \geq c
	\]
		Then there exists $D \in \Delta(\Delta_{\cH})$ such that for any consistent and proper learning algorithm $\left(A_m\right)_{m\in M}$ it holds that
	\[
	\EE{p \sim D} \: \EE{m \sim \mu} \: \I{S \sim p^m}{A_m(S);S} \geq c
	\]
\end{lemma}


{The second direct sum theorem assumes that the consequent of Lemma~\ref{lemma-global-distribution-extended} holds, i.e.\ that for a specific distribution on $m$ (the number of examples received), each class has a hard distribution on distributions such that the expected mutual information over all those distributions is high for all algorithms. }


\begin{theorem}[\textbf{Direct sum II}]\label{lemma_additivity}
	Let $m,d \in \mb{N}$ and $T= [\frac{m}{2},\frac{3m}{2}] \cap \mathbb{N}$. {Let $\mu\in\Delta(T)$ be the distribution\footnote{{$\mr{Bin}(k,p)$ is the binomial distribution with $k$ i.i.d.\ trials each of which has probability of success $p$.}} $\mr{Bin}(dm,\frac{1}{d})$ conditioned on the event
	$E$ that the integer sampled is in $T$.} Assume that for each $i\in[d]$:
	\begin{enumerate}
		\item{$\cH_i \subseteq \cY_i^{\cX_i}$ is a class of hypotheses.}
		\item{\label{assumption-cost-for-class}
			There exists $c_i\in\mb{R}$ and $D_i\in\Delta(\Delta_{\cH_i})$ such that for every (possibly randomized) consistent and proper learning algorithm $\left(A_t\right)_{t\in T}$ for $\cH_i$, it holds that
			\[
			\EE{p \sim D_i} \: \EE{t \sim \mu} \: \I{S \sim p^t}{A_t(S);S} \geq c_i
			\]}
	\end{enumerate}
	Finally, let $\cH$ be the product of $\cH_1,\dots,\cH_d$. Then for every (possibly randomized) consistent and proper learning algorithm $A$ for $\cH$ that accepts samples of size $dm$ there exists a distribution $p\in\Delta_\cH$ such that
	\[
	\I{S \sim p^{dm}}{A(S);S} \geq \alpha \sum_{i=1}^d c_i
	\]
	where $\alpha = 1-2e^{-\frac{m}{2d}}$.
\end{theorem}

\section{The Lower Bound for VC Classes}\label{VC bound}

To prove Theorem~\ref{general-lower-bound} we need to verify that condition 2 of Theorem~\ref{lemma_additivity} is true for {the product of $d$ thresholds.}

\begin{definition}
	Let $\cT_{n,d}$ be {the $d$-fold product of $\cT_n$.
	The domain of the functions in $\cT_{n,d}$ is $[d \cdot 2^n]$.
	Each function in $\cT_{n,d}$ is of the form:
	$$
	f_{k_1,\dots k_d}(x) = 
		\begin{cases}
		0 & x < k_{i(x)} \\
		1 & x \geq k_{i(x)}
		\end{cases} ~~~~~~~~	i(x) = \left\lceil \frac{x}{2^d} \right\rceil 
	$$
	where $k_j \in [(j-1)\cdot2^n+1, j\cdot2^n+1]$ for each $j\in [d]$.
	(See Figure~\ref{figure-Tnd} for a graphical illustration of this definition.)}
\end{definition}

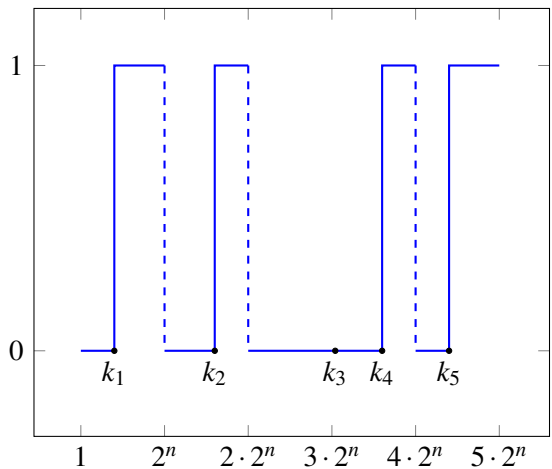
\begin{figure}[h]
	\centering
	\begin{tikzpicture}
	\begin{axis}[%
	,ytick={0,1}
	,xtick={0,5,10,15,20,25}
	,xticklabels={$1$,$2^n$,$2\cdot2^n$,$3\cdot2^n$,$4\cdot2^n$,$5\cdot2^n$}
	,ymax=1.2 
	,xmax=28
	,ymin=-0.3
	]
	\addplot+[const plot, no marks, thick, blue] coordinates {(0, 0) (2, 0) (2, 1) (5,1)} 		node[above,pos=.57,black] {};
	\addplot+[const plot, no marks, thick, blue] coordinates {(5, 0) (8, 0) (8, 1) (10,1)} 		node[above,pos=.57,black] {};
	\addplot+[const plot, no marks, thick, blue] coordinates {(10, 0) (15, 0)} 					node[above,pos=.57,black] {};
	\addplot+[const plot, no marks, thick, blue] coordinates {(15, 0) (18, 0) (18, 1) (20,1)} 	node[above,pos=.57,black] {};
	\addplot+[const plot, no marks, thick, blue] coordinates {(20, 0) (22, 0) (22,1) (25,1)} 	node[above,pos=.57,black] {};
	
	\addplot+[dashed, no marks, thick, blue] coordinates {(5, 1) (5, 0)} 						node[above,pos=.57,black] {};
	\addplot+[dashed, no marks, thick, blue] coordinates {(10, 1) (10, 0)} 						node[above,pos=.57,black] {};
	\addplot+[dashed, no marks, thick, blue] coordinates {(20, 1) (20, 0)} 						node[above,pos=.57,black] {};	
	
	\filldraw (2,0) circle (1pt) node[align=left,   below] {$k_1$};
	\filldraw (8,0) circle (1pt) node[align=left,   below] {$k_2$};
	\filldraw (15.2,0) circle (1pt) node[align=left,   below] {$k_3$};
	\filldraw (18,0) circle (1pt) node[align=left,   below] {$k_4$};
	\filldraw (22,0) circle (1pt) node[align=left,   below] {$k_5$};
	\end{axis}
	\end{tikzpicture}
	\caption{\label{figure-Tnd}Example of a function in $\cT_{n,5}$.}
\end{figure}

\nin Lemma~\ref{clm:VCdimofX} implies: 

\begin{corollary}\label{vc-dim-of-t_nd}
	The VC dimension of $\cT_{n,d}$ is $d$.
\end{corollary}

\nin It thus remains to prove the following {lemma,
which is stronger than Theorem~\ref{lower-bound-for-thresholds}.}
 
\begin{lemma}\label{lemma-lower-bound-for-threshold-randomized-sample-size}
	Let $a,b,n\in\mb{N}$, $2\leq a\leq b$, $M=[a,b] \cap \mathbb{N}$, let $\left(A_m\right)_{m\in M}$ be a (possibly randomized) consistent and proper learning algorithm for $\cT_n$, and let $\mu\in\Delta(M)$. Then there exist a distribution $p\in\Delta_{\cT_n}$ such that
	\[
	\EE{m \sim \mu} \: \I{S \sim p^m}{S;A_m(S)} = \Omega\left(\left(\frac{a}{b}\right)^2 \log n\right).
	\]
\end{lemma}

\nin Finally, we can prove the theorem for VC classes.\\

\begin{proof}[Theorem \ref{general-lower-bound}].
	{Fix some $m \geq 2d$ and take $\mu$ as in Theorem~\ref{lemma_additivity}. Apply Lemma~\ref{lemma-lower-bound-for-threshold-randomized-sample-size} with $a=\frac{m}{2}$ and $b=\frac{3m}{2}$, which together with Lemma~\ref{lemma-global-distribution-extended} entails that there exists $D \in \Delta(\Delta_{\cT_n})$ such that for any $\left(A_m\right)_{m\in M}$ for $\cT_n$,
	\[
	\EE{p \sim D} \: \EE{m \sim \mu} \: \I{S \sim p^m}{A_m(S);S} \geq  \Omega\left(\log n\right).
	\]
	Theorem~\ref{general-lower-bound} now follows directly from applying Theorem~\ref{lemma_additivity} to the class $\cT_{n,d}$.}
\end{proof}

\section{Discussion and Directions for Further Work}

A direct continuation of the current line of research would be to extend the lower bound for VC classes to PAC learners that are not necessarily proper or consistent. Note that the lower bound for thresholds does not hold for the case of randomized, consistent, non-proper algorithms. Consider the algorithm for thresholds that outputs a hypothesis $h$ as follows. For any $(x,y)$ in the training sample, $h(x)=y$. For any $x$ that did not appear in the sample, $h(x)$ is sampled uniformly from $\{0,1\}$. This algorithm has mutual information that does not grow with the size of the domain (it is $O(m)$). This is not too meaningful, as this algorithm is not a PAC learner. But it illustrates that the lower bound breaks somewhere, and it would be worthwhile to identify exactly how far the assumptions can be pushed before it breaks.

A different and interesting direction is to prove upper bounds on information complexity. First, we would like to understand whether the lower bound presented here is sharp. Better yet: Can we provide explicit general constructions for learning algorithms that obtain the information complexity, i.e., retain the minimal amount of information possible? Following the theorem of \cite{bassily2018learners} stating that compression entails learning, this would yield a novel class of learning algorithms for all hypothesis classes in which the information complexity is $o(m)$ -- a strong result that might even have practical applications.

Lastly, and perhaps most interestingly, one may also consider the converse of that theorem: Is there a sense in which low information complexity is a necessary condition for learnability? Are the concepts equivalent?

\bibliography{references}

\appendix
 
\section{Proofs}

\subsection{Lower Bound for Thresholds}

\label{sec:LBforT}

We start with two lemmas from \cite{bassily2018learners}.

\begin{lemma}\label{non_deteministic_induction_lemma}
	Let $Q\in\mathrm{Mat}_{2^{n} \times 2^{n}} (\Delta([2^n]))$, i.e.\ $Q$ is a $2^{n} \times 2^{n}$ matrix where each cell contains a probability function over $[2^n]$. Furthermore, assume that $Q$ is symmetric and that it has the property that
	\[
	\forall i,j:\ \supp\: Q_{ij} \subseteq [\min\{i,j\}, \max\{i,j\}] \tag{ii}
	\]
	Then $Q$ contains a row with $n+1$ distributions $p_1, \dots, p_{n+1}$ such that there exist pairwise disjoint sets $S_1, \dots,S_{n+1} \subseteq [2^n]$ satisfying
	\[
	\forall i\in[n+1]:\ p_i(S_i) \geq \frac{1}{2} \tag{iii}
	\]
	(And hence, from symmetry, it also contains such a column.)
\end{lemma}

\begin{lemma}\label{lemma_probabilities_extended}
	Let $\Omega$ be a finite sample space (a finite set), let $p_1, \dots,p_n$ be probability distributions over $\Omega$, and let $A_1, \dots, A_n \subseteq \Omega$ be pairwise disjoint events such that
	\[
	\forall i \in [n]:\ p_i(A_i)\geq \frac{1}{2} \tag{v}
	\]
	Let $U$ be a random variable distributed uniformly over $[n]$, and let $W$ be
	\[
	W \sim p_i, i \sim U
	\] 
	Namely, $W$ is a random variable over $\Omega$ that results from sampling an index $i\in[n]$ according to $U$ and then sampling an element of $\Omega$ according to $p_i$ and assigning that element to $W$.
	
	Then it holds that
	\[
	\II{U; W} = \Omega(\log n)
	\]
\end{lemma}

\nin Now, we can prove the desired lower bound for probabilistic algorithms using lemma \ref{permutation-decreases-information}.\\

\begin{proof}[Theorem~\ref{lower-bound-for-thresholds}].
	Let $A$ be a consistent learning algorithm for $\cT_n$, and let $A'$ be as in Lemma~\ref{permutation-decreases-information}. Let
	\[
	M \in \mathrm{Mat}_{2^n\times2^n}(\Delta([2^n]))
	\]
	be a matrix such that for all $i\neq j$, $M_{ij}$ is a probability function such that $M_{ij}(k)$ is the probability that $A'$ will output hypothesis $f_k$ for an input sample of the form
	\[
	S_{ij} = \Big(\underbrace{(1,0), \dots, (1,0)}_{\text{ m-2 }}, (\min\{i,j\},0), (\max\{i,j\},1) \Big)
	\]
	and for all $i$, $M_{ii}$ is the degenerate distribution that assigns probability 1 to $i$. Notice that because $A'$ is indifferent to the order of the examples in the input sample, $M_{ij}$ actually equals the probability functions of the output for any permutation of $S_{ij}$.
	
	$M$ is symmetric and because $A'$ is consistent it follows that $M$ satisfies property (ii), and hence by lemma \ref{non_deteministic_induction_lemma} $M$ contains a row $r$ with probabilities $p_1,\dots,p_{n+1}$ for which there are pairwise disjoint sets $A_1,\dots,A_{n+1} \subseteq [2^n]$ such that for all $i$, $p_i(A_i) \geq \frac{1}{2}$. Note that at least $\frac{n}{2}$ of these probabilities on row $r$ are located above the diagonal, or else, from symmetry of $M$, at least $\frac{n}{2}$ of them are located above the diagonal on column $r$. Thus, we assume w.l.o.g.\ that probabilities $p_1,\dots,p_\frac{n}{2}$ are located above the diagonal on row $r$ in cells $(r,k_1),\dots,(r,k_{\frac{n}{2}})$ (the symmetric case can be handled very similarly).
	
	We use the following probability $p$ over realizable samples of length $m$, where $U_K$ is the uniform distribution over $\{k_1,\dots,k_{\frac{n}{2}}\}$:
	\[
	p(w) = \left(1-\frac{1}{m}\right)\textbf{1}_{w=1}(w) + \frac{1}{2m}\textbf{1}_{w=r}(w) + \frac{1}{2m} U_K(w)
	\]
	
	\nin Consider the event in which the generated sample $S$ is any permutation of
	\[
	S_{rk_i} = \Big(\underbrace{(1,0), \dots, (1,0)}_{\text{ m-2 }}, (r,0), (k_i,1) \Big)
	\]
	for some $k_i$, and let $E$ be an indicator random variable of this event. $p$ satisfies
	\[
	p(E=1) \geq \Big(1-\frac{1}{m}\Big)^{m-2}\cdot \Big(\frac{1}{2m}\Big)^2\cdot m(m-1) \geq \frac{1}{16e}
	\]
	and
	\[
	p(S_{rk_1})=p(S_{rk_2})=\cdots=p(S_{rk_{\frac{n}{2}}})
	\]
	
	Let $h$ be a random variable denoting the output of $A$ when the input sample $S$ is distributed according to $p$. We have the following chain of inequalities
	\[
	\II{S;h} \stackrel{(a)}{\geq} \II{S;h|E} \stackrel{(b)}{=} p(E=1)\cdot \II{S;h|E=1} + p(E=0)\cdot \II{S;h|E=0} \stackrel{(c)}{\geq}
	\]
	\[
	\stackrel{(c)}{\geq} p(E=1)\cdot \II{S;h|E=1} \geq \frac{1}{16e}\cdot \II{S;h|E=1} \stackrel{(d)}{\geq} 
	\]
	\[
	\stackrel{(d)}{\geq} \frac{1}{16e}\cdot \II{U(S);h|E=1} \stackrel{(e)}{=} \Omega(\log n)
	\]
	which is justified as follows:
	\begin{enumerate}[label=(\alph*)]
		\item {Notice that $E \bot h | S$ because once we saw the actual sample, we know with certainty whether event $E$ occurred or not (formally, $\II{E;h|S} \leq \HH(E|S) = 0 \ \Longrightarrow \ E \bot h | S$). Thus, this inequality follows from claim \ref{conditional-information-inequality}.\ref{cii-conditional-independence} in the appendix.}
		\item {Definition of conditional mutual information.}
		\item {Positivity of mutual information.}
		\item {Here $U(\cdot)$ is any mapping that satisfies $\sigma(S_{rk_i}) \mapsto i$ for all $i$ and all permutations $\sigma$. The inequality then follows from the data processing inequality.}
		\item {Given that $E=1$, $U(S)$ is the uniform distribution on $[\frac{n}{2}]$. Furthermore, $h$ is the result of sampling a hypothesis according to the distribution $p_i$, where $i$ is the value of $U(S)$. Lastly, our choice of $p_1,\dots,p_{\frac{n}{2}}$ ensured that there exist pairwise disjoint sets $A_1,\dots,A_{\frac{n}{2}}$ such that $p_i(A_i)\geq \frac{1}{2}$ for all $i$, and so the lower bound follows from lemma \ref{lemma_probabilities_extended}.}
	\end{enumerate}
	
	Thus, we have shown that for every consistent learning algorithm for $\cT_n$ that accepts samples of size $m$ there exists a distribution $p\in\Delta_{\cT_n}$ such that
	\[
	\I{S \sim p^m}{S;h} = \Omega(\log \log |\cX|)
	\]
	as desired.
\end{proof}

\subsection{Direct Sum}

\begin{proof}[Lemma~\ref{clm:VCdimofX}].
	Let $\cH$ be the product class, $k=\VC(\cH)$ and $k_i=\VC(\cH_i)$ for all $i$. To see that $k \geq \sum k_i$, take sets $R_i \subset \cX_i$ for all $i$ such that $|R_i|=k_i$ and $\cH_i$ shatters $R_i$ (such sets exist because $k_i=\VC(\cH_i)$). Now note that $\cH$ shatters $\bigcup_{i=1}^d R_i$. 
	
	To see that $k \leq \sum k_i$, assume for contradiction that $\cH$ shatters a set $R$ of size strictly more than $\sum k$. Then there exists $j$ such that $|\cX_j\cap R|>k_j$. The assumption entails that $\cH_j$ shatters $\cX_j\cap R$, a contradiction.
\end{proof}

\begin{proof}[Theorem~\ref{superadditivity}].
	Denote $M=(M_1,\dots,M_d)$. Then
	\begin{align*}
	\II{A(S);S} & \stackrel{(a)}{=} \II{A(S);S,M} 
	 \stackrel{(b)}{\geq} \II{A(S);S\:|\:M} 
	 \stackrel{(c)}{\geq} \II{A(S);S_1,\dots,S_d\:|\:M} 
	 \stackrel{(b)}{=} \sum_{i=1}^d\II{A(S);S_i\:|\:M,S_{<i}} \\
	& = \sum_{i=1}^d\II{h_1,\dots,h_d;S_i\:|\:M,S_{<i}} 
	 \stackrel{(b)}{\geq} \sum_{i=1}^d\II{h_i;S_i\:|\:M,S_{<i}} 
	 \stackrel{(d)}{\geq} \sum_{i=1}^d\II{h_i;S_i\:|\:M} 
	 \stackrel{(e)}{\geq} \sum_{i=1}^d\II{h_i;S_i\:|\:M_i} 
	\end{align*}
	
	where the steps are justified as follows:
	\begin{enumerate}[label=(\alph*)]
		\item{$(S,M)$ and $S$ are functions of each other.}
		\item{The chain rule for mutual information.}
		\item{$S_1,\dots,S_d$ is a function of $S$ (data processing inequality).}
		\item{Follows from claim \ref{conditional-information-inequality}.\ref{cii-independence} because $S_i \bot S_{<i} \: | \: M_i$.}
		\item{Again from claim \ref{conditional-information-inequality}.\ref{cii-independence}, because $S_i \bot M_{\neq i} \: | \: M_i$.}
	\end{enumerate} 
	And the proof is complete.
\end{proof}

\begin{proof}[Theorem~\ref{lemma_additivity}].
	Let $D$ be the distribution on distributions that results from sampling a distribution $p_i$ from $D_i=D_{i_\mu}$ for each $i\in[d]$, and then taking the average of these distributions. Formally, $D \in \Delta(\Delta_\cH)$ is defined as follows:
	\[
	D(p) =  \begin{cases}
	\prod_{i=1}^d D_i(p_i) & p= \frac{1}{d}\sum_{i=1}^d p_i \text{ s.t. } \forall i: p_i \in \Delta_{\cH_i} \\
	0 & \text{otherwise}
	\end{cases}
	\]
	
	Taking expectation on both sides of Lemma~\ref{superadditivity}, we have that
	\begin{align*}
	\EE{p \sim D} \: \I{S \sim p^{dm}}{A(S);S} & \geq \EE{p \sim D} \: \sum_{i=1}^d \I{S \sim p^{dm}}{h_i;S_i|M_i} 
	 = \sum_{i=1}^d \EE{p \sim D} \: \I{S \sim p^{dm}}{h_i;S_i|M_i} 
	 \end{align*} 
	 
	 \begin{align*}
	& = \sum_{i=1}^d \EE{p_1 \sim D_1}\cdots\EE{p_d \sim D_d} \: \I{S \sim \left(\frac{1}{d}\sum_{j=1}^d p_j \right)^{dm}}{h_i;S_i|M_i} 
	 = \sum_{i=1}^d \EE{p_{\neq i} \sim D_{\neq i}} \: \EE{p_i \sim D_i} \: \I{S \sim \left(\frac{1}{d}\sum_{j=1}^d p_j \right)^{dm}}{h_i;S_i|M_i} 
	\end{align*}
	where $\EE{p_{\neq i} \sim D_{\neq i}}$ is a shorthand notation for
	\[
	\EE{p_1 \sim D_1}\cdots\EE{p_{i-1} \sim D_{i-1}} \: \EE{p_{i+1} \sim D_{i+1}}\cdots\EE{p_d \sim D_d}
	\]
	
	Next, we bound the innermost expectation on $p_i$ for any fixed vector of distributions $p_{\neq i}$. Let $E_i$ be the event in which $M_i \in T$. Then
	$$
	\EE{p_i \sim D_i} \: \I{S \sim \left(\frac{1}{d}\sum_{j=1}^d p_j \right)^{dm}}{h_i;S_i|M_i}  \stackrel{(a)}{=} \EE{p_i \sim D_i} \: \I{S \sim \left(\frac{1}{d}\sum_{j=1}^d p_j \right)^{dm}}{h_i;S_i|M_i,1_{E_i}}  $$
	$$ \geq \EE{p_i \sim D_i} \Pr(1_{E_i}=1) \I{S \sim \left(\frac{1}{d}\sum_{j=1}^d p_j \right)^{dm}}{h_i;S_i|M_i,1_{E_i}=1} \\
	 \stackrel{(b)}{\geq} \alpha \:  \EE{p_i \sim D_i} \: \I{S \sim \left(\frac{1}{d}\sum_{j=1}^d p_j \right)^{dm}}{h_i;S_i|M_i,1_{E_i}=1} $$
$$	 = \alpha \: \EE{p_i \sim D_i} \: \EE{m_i \sim M_i|1_{E_i}=1} \: \I{S \sim \left(\frac{1}{d}\sum_{j=1}^d p_j \right)^{dm}}{h_i;S_i|M_i=m_i} 
	 \stackrel{(c)}{=} \alpha \: \EE{p_i \sim D_i} \: \EE{m_i \sim \mu} \: \I{S \sim \left(\frac{1}{d}\sum_{j=1}^d p_j \right)^{dm}}{h_i;S_i|M_i=m_i} 
	 \stackrel{(d)}{\geq} \alpha c_i
$$
	which is justified as follows:
	\begin{enumerate}[label=(\alph*)]
		\item{$M_i$ and $(M_i,1_{E_i})$ are functions of each other.}
		\item{$\Pr(1_{E_i}=1)\geq\alpha$, from Claim~\ref{claim-chernoff}.}
		\item{$\mu = (M_i|1_{E_i}=1)$}
		\item{From assumption \ref{assumption-cost-for-class}. Note: $A$ takes an input sample $S$ of which $S_i$ is just a subsample, and outputs a vector of hypotheses of which $h_i$ is just one component. However, we may ignore these other outputs, and we may regard the other input subsamples $S_j$ for $j\neq i$ as random coins used by $A$. Thus, for the sake of this analysis $A$ is viewed as a randomized learning algorithm that takes $S_i\sim p_i^{m_i}$ as input and produces $h_i$ as output.}
	\end{enumerate}
	
	Thus, we have
	$
	\EE{p \sim D} \: \I{S \sim p^{dm}}{A(S);S} \geq \alpha \sum_{i=1}^d c_i
	$. This entails that there exists a distribution $p\in\supp(D)$ such that
	$
	\I{S \sim p^{dm}}{A(S);S} \geq \alpha \sum_{i=1}^d c_i
	$
	as desired.
\end{proof}
\subsection{Proofs for Section \ref{VC bound}}

\begin{proof}[Lemma~\ref{lemma-global-distribution-extended}]
	For each $m\in M$, let $\cA_m$ be the set of consistent learning algorithms for $\cH$ that accept samples of size $m$, and let
	\[
	\cA = \Pi_{m \in M} \cA_m
	\]
	Notice that
	\begin{align*}
	c & \leq \inf_{A\in\cA}\sup_{p\in\Delta_{\cH}} \EE{m \sim \mu} \: \I{S \sim p^m}{A_m(S);S} \\
	& \leq \inf_{A\in\cA}\sup_{D\in\Delta(\Delta_{\cH})} \EE{p \sim D} \: \EE{m \sim \mu} \: \I{S \sim p^m}{A_m(S);S}
	\end{align*}
	where the first inequality follows from the assumption and the second holds because $\Delta(\Delta_{\cH})$ contains all the degenerate distributions that assign probability $1$ to a single distribution in $\Delta_{\cH}$. We now choose topologies in which the assumptions of Sion's theorem (theorem \ref{sion-minimax}) are satisfied:
	\begin{itemize}
		\item{
			$\cA$ is convex, and it is compact in $\mb{R}^k$ for a finite $k$. Every randomized algorithm can be identified with a conditional probability function $p(h|s)$. For each realizable sample with length in $M$, the algorithm assigns a point in $\Delta(\cH)$ (not to be confused with $\Delta_\cH$). Thus, the set $\cA$ of all algorithms is the product of $t$ simplices, each of finite dimension $|\cH|$, where $t$ is the number of such realizable samples. We conclude that $\cA$ is a compact and convex subset of $\mb{R}^k$, for $k=t\cdot|\cH|$. It will be convenient to view $\mb{R}^k$ as the metric space induced by the $\ell_1$ norm.
		}
		\item{
			$\Delta(\Delta_{\cH})$ is convex. This is immediate, seeing that if $D_1,D_2 \in \Delta(\Delta_{\cH})$, then $|\supp(D_1)|$, $|\supp(D_2)| < \infty$ and therefore
			\[
			\left|\supp\Big(\lambda D_1 + (1-\lambda) D_2 \Big)\right| < \infty
			\]
			Topologically, we view $\Delta(\Delta_{\cH})$ as a metric space with the metric induced by the $\ell_1$ norm. 
		}
		\item{
			The function $f(A, D) = \EE{p \sim D} \: \EE{m \sim \mu} \: \I{S \sim p^m}{A(S);S}$ is continuous with respect to the product topology induced on the domain. We view the domain as the metric space induced by the $\ell_1$-norm product metric (which induces the product topology). Fix some $D_0 \in \Delta(\Delta_\cH)$, $A_0=p_0(h|s)\in\cA$ and $\varepsilon > 0$. We will find a value $\delta>0$ such that
			\[
			\|(A,D)-(A_0,D_0)\|_1<\delta \: \Longrightarrow \: |f(A, D)-f(A_0, D_0)|\leq \varepsilon
			\]
			
			Consider $g: \cA \times \Delta_\cH \rightarrow \mathbb{R}$ as follows:
			\[
			g(A,q)=\EE{m \sim \mu}\I{S\sim q^m}{A(S);S}= \EE{m \sim \mu} \sum_{s \in \cS^m}q^m(s) \sum_{h \in \cH} p(h|s) \log \frac{p(h|s)}{\sum_{s \in \cS^m} p(h|s)q^m(s)}
			\]
			clearly, $g$ is continuous with respect to $(p,q)$, and because $\cA \times \Delta_\cH$ is compact, $g$ is uniformly continuous (per the Heine--Cantor theorem). Take $\delta'>0$ such that
			\[
			\|(A_1,q_1)-(A_2,q_2)\|_1<\delta' \: \Longrightarrow \: |g(A_1,q_1)-g(A_2,q_2)|<\frac{\varepsilon}{2}
			\]
			Now, taking $\delta=\min\{\delta', \frac{\varepsilon}{2\log|\cH|}\}$ we obtain
			\[
			|f(A,D)-f(A_0,D_0)| \leq |f(A,D)-f(A_0,D)|+|f(A_0,D)-f(A_0,D_0)| =
			\]
			\[
			= \left|\EE{p \sim D} \left(g(A,p)-g(A_0,p)\right)\right| + \left|\sum_{p} \left(D(p)-D_0(p)\right)g(A,p)\right| \leq
			\]
			\[
			\leq \EE{p \sim D} \left|g(A,p)-g(A_0,p)\right| + \log |\cH|\sum_{p} \left|D(p)-D_0(p)\right| \leq
			\]
			\[
			\leq \frac{\varepsilon}{2} + \log|\cH| \cdot \frac{\varepsilon}{2\log|\cH|} \leq \varepsilon
			\]
			as desired.
		}
		\item{The function $f(A, D) = \EE{p \sim D} \: \EE{m \sim \mu} \: \I{S \sim p^m}{A(S);S}$ is convex-concave
			\begin{itemize}
				\item{$f$ is convex in $A$ (for fixed $D$). This follows from Lemma~\ref{cover06-concave} where we take $X$ to be $S$ and $Y$ to be $A(S)$. We can identify the set of algorithms with the set of conditional probabilities $p(y|x)$. The lemma tells us that for each $p$ in $\supp(D)$, the mutual information is convex, which entails that the expectation is also convex.}
				\item{$f$ is concave in $D$ (for fixed $A$). In fact $f$ is linear in $D$, from the linearity of expectation.}
			\end{itemize}
		}
	\end{itemize}
	
	Thus, the assumptions for Sion's minimax theorem hold, and we obtain that
	\[
	\inf_{A\in\cA}\sup_{D\in \Delta(\Delta_{\cH})} \EE{p \sim D} \: \EE{m \sim \mu} \: \I{S \sim p^m}{A_m(S);S} = \sup_{D\in \Delta(\Delta_{\cH})}\inf_{A\in\cA} \EE{p \sim D} \: \EE{m \sim \mu} \: \I{S \sim p^m}{A_m(S);S}
	\]
	as desired.
\end{proof}

\begin{proof}[Lemma~\ref{lemma-lower-bound-for-threshold-randomized-sample-size}]
	We define $L$ to be a consistent randomized learning algorithm for $\cT_n$ that accepts samples of size $2$ as follows. Let $E$ denote the event in which $L$ receives a sample $S$ of the form $\big((i,0),(j,1)\big)$ or $\big((j,1),(i,0)\big)$, i.e.\ $S$ contains precisely one example $i$ that is labeled with $0$ and one example $j$ that is labeled $1$.
	\begin{itemize}
		\item{If $E$ occurs, then $L$ samples an integer $m$ from $\mu$, samples a permutation $\sigma$ uniformly from all permutations on $m$ elements, and returns the hypothesis
			\[
			A_m\Big(\sigma\big(\underbrace{(i,0),\dots,(i,0)}_{\text{ m-1 }},(j,1)\big)\Big)
			\]
		}
		\item{Otherwise, $L$ returns some arbitrary hypothesis that is consistent with $S$.}
	\end{itemize}
	
	from the proof of Lemma~\ref{lower-bound-for-thresholds}, there exists a distribution $q\in\Delta_{\cT_n}$ such that
	\[
	\I{S \sim q^2}{S;L(S) \:|\: 1_E=1} = \Omega(\log n)
	\]
	and we can assume without loss of generality that $(S|1_E=1)$ is such that the value of $i$ is fixed and $j$ is distributed uniformly over some set of size $\Omega(n)$. We use $f$ to denote the mapping 
	\[
	(j, m, \sigma) \longmapsto \sigma\big(\underbrace{(i,0),\dots,(i,0)}_{\text{ m-1 }},(j,1)\big)
	\]
	$J$ to denote the value of $j$ in $S$, $U_J$ for the uniform distribution on the set of values for $j$, $U_{\Sigma,M}$ for the uniform distribution on the $M$ orderings of a sample of this form with $M$ elements, and we use $B$ to denote a bit indicating whether $j$ appeared first or second in $S$. We now may write
	\begin{align*}
	\I{S \sim q^2}{S;L(S) \:|\: 1_E=1} & = \I{S \sim q^2}{J,B;L(S) \:|\: 1_E=1} \\
	& \leq \I{S \sim q^2}{J;L(S) \:|\: 1_E=1} + 1 \\
	& = \I{S \sim q^2\:|\: 1_E=1}{J;L(S)} + 1 \\
	& = \I{\substack{M \sim \mu \\ J \sim U_J \\ \Sigma \sim U_{\Sigma,M}}}{J;A_M(f(J,M,\Sigma))} + 1
	\end{align*}
	where the inequality follows from the chain rule. It holds that
	\begin{align*}
	\I{\substack{M \sim \mu \\ J \sim U_J \\ \Sigma \sim U_{\Sigma,M}}}{J;A_M(f(J,M,\Sigma))} 
	& \stackrel{(a)}{\leq} \II{J;A_M(f(J,M,\Sigma)) \: | \: M} \\ 
	& \stackrel{(b)}{\leq} \II{J,M,\Sigma;A_M(f(J,M,\Sigma)) \: | \: M} \\
	& \stackrel{(c)}{=} \II{f(J,M,\Sigma);A_M(f(J,M,\Sigma)) \: | \: M} \\
	& = \EE{m \sim \mu} \: \I{\substack{J \sim U_J \\ \Sigma \sim U_{\Sigma,m}}}{f(J,M,\Sigma);A_M(f(J,M,\Sigma)) \: | \: M=m} \\
	\end{align*}
	which is justified by:
	\begin{enumerate}[label=(\alph*)]
		\item{From Lemma~\ref{conditional-information-inequality}.\ref{cii-independence}, because $J \bot M$.}
		\item{The data processing inequality.}
		\item{Because $f$ is a bijection (data processing inequality).}
	\end{enumerate}
	
	Recall that $q$ is defined as follows, for $t=2$: 
	\[
	q(w) = \left(1-\frac{1}{t}\right)\textbf{1}_{w=1}(w) + \frac{1}{2t}\textbf{1}_{w=i}(w) + \frac{1}{2t} U_J(w)
	\]
	Let $p$ be the same distribution, but with $t=b$. Then:
	\[
	\EE{m \sim \mu} \: \I{\substack{J \sim U_J \\ \Sigma \sim U_{\Sigma,m}}}{f(J,M,\Sigma);A_m(f(J,M,\Sigma)) \: | \: M=m} 
	\]
	\begin{align*}
	\:\:\:\:\:\:\:\:\:\: & \stackrel{(a)}{=} \EE{m \sim \mu} \: \I{S \sim p^m}{S;A_m(S) \: | \: M=m, 1_{E_m}=1} \\
	& \stackrel{(b)}{\leq} \EE{m \sim \mu} \: \frac{1}{\Pr(1_{E_m}=1)} \I{S \sim p^m}{S;A_m(S) \: | \: M=m,1_{E_m}} \\
	& \leq \left(\max_{m}\frac{1}{\Pr(1_{E_m}=1)}\right) \:\EE{m \sim \mu} \: \I{S \sim p^m}{S;A_m(S) \: | \: M=m,1_{E_m}} \\
	& \stackrel{(c)}{\leq} 16e \left(\frac{b}{a}\right)^2 \:\EE{m \sim \mu} \: \I{S \sim p^m}{S;A_m(S) \: | \: M=m,1_{E_m}} \\
	& \stackrel{(d)}{\leq} 16e \left(\frac{b}{a}\right)^2 \left(\EE{m \sim \mu} \: \I{S \sim p^m}{S;A_m(S) \: | \: M=m} + 1\right) \\
	\end{align*}
	where $E_m$ is the event that the sample $S$ is a permutation of
	\[
	\Big(\underbrace{(1,0), \dots, (1,0)}_{\text{ m-2 }}, (i,0), (j,1) \Big)
	\]
	for some $i<j$, and the justifications are:
	\begin{enumerate}[label=(\alph*)]
		\item{$(S|1_{E_m}=1)$ and $f(J,M,\Sigma)$ have the same distribution.}
		\item{From the definition of conditional mutual information.}
		\item{
			From the construction of $p$, we have for $m\in[a,b]$ that
			\[
			\Pr(1_{E_m}=1) = \left(1-\frac{1}{b}\right)^{m-2}\left(\frac{1}{2b}\right)^2m(m-1) \geq
			\]
			\[
			\geq \left(1-\frac{1}{b}\right)^{b-2}\left(\frac{1}{2b}\right)^2a(a-1)\geq \frac{1}{16e} \left(\frac{a}{b}\right)^2
			\]
		}
		\item{The chain rule for mutual information.}
	\end{enumerate}
	
	 Finally, chaining all the above inequalities together yields
	\[
	\EE{m \sim \mu} \: \I{S \sim p^m}{S;A_m(S)} \geq \left(\frac{a}{b}\right)^2 \frac{\Omega(\log n) - 1}{16e} - 1 = \Omega\left(\left(\frac{a}{b}\right)^2 \log n\right)
	\]
	as desired.
\end{proof}

\section{Miscellaneous}

\begin{claim}\label{claim-chernoff}
Assume $dm$ integers are sampled i.i.d.\ from the uniform distribution on $[d]$, and let $Z_i$ denote the number of times the integer $i\in[d]$ was sampled. Then 
\[
\Pr\left(\frac{m}{2}\leq Z_i \leq \frac{3m}{2}\right) \geq 1-2e^{-\frac{m}{2d}}
\]
\end{claim}

\begin{proof}[Claim~\ref{claim-chernoff}]
	Let $X_t$ be an indicator denoting whether the $t$-th integer sampled was $i$.
	\begin{align*}
	\Pr\left(\frac{m}{2}\leq Z_i \leq \frac{3m}{2}\right) & = 1-\Pr\left(\left| \sum_{t=1}^{dm} X_t - \E \sum_{t=1}^{dm} X_t \right|> \frac{m}{2}\right) \\
	& = 1-\Pr\left(\left| \sum_{t=1}^{dm} X_t - m \right|> \frac{m}{2}\right) \\
	& = 1-\Pr\left(\left| \frac{1}{dm}\sum_{t=1}^{dm} X_t - \frac{1}{d} \right|> \frac{1}{2d}\right)
	\end{align*}
	And from Hoeffding's inequality \citep[lemma B.6 in][]{shalev2014understanding}
	\begin{align*}
		\Pr\left(\left| \frac{1}{dm}\sum_{t=1}^{dm} X_t - \frac{1}{d} \right|> \frac{1}{2d}\right) & \leq 2e^{-2dm\left(\frac{1}{2d}\right)^2} = 2e^{-\frac{m}{2d}}
	\end{align*}
	as desired.~$\blacksquare$
\end{proof}
 
\begin{claim}\label{conditional-information-inequality}
 	Let $X,Y,Z$ be random variables. 
 	\begin{enumerate}
 		\item{\label{cii-independence} If $X\bot Z$ then $\II{X;Y} \leq \II{X;Y|Z}$.}
 		\item{\label{cii-conditional-independence} If $X\bot Z | Y$ then $\II{X;Y} \geq \II{X;Y|Z}$.}
 	\end{enumerate}
\end{claim}
 
\begin{proof}[Claim~\ref{conditional-information-inequality}]\\
 	For \ref{cii-independence}:
 	\begin{align*}
 	\II{X;Y|Z} & = \HH(X|Z) - \HH(X|Y,Z) \\
 	& \stackrel{(*)}{=} \HH(X) - \HH(X|Y,Z) \\
 	& \geq \HH(X) - \HH(X|Y) \\ 
 	& = \II{X;Y}
 	\end{align*}
 	For \ref{cii-conditional-independence}:
 	\begin{align*}
 	\II{X;Y|Z} & = \HH(X|Z) - \HH(X|Y,Z) \\
 	& \leq \HH(X) - \HH(X|Y,Z) \\
 	& \stackrel{(*)}{=} \HH(X) - \HH(X|Y) \\ 
 	& = \II{X;Y}
 	\end{align*}
 	where $(*)$ follow from the assumptions.~$\blacksquare$
\end{proof}
 
\begin{lemma}\label{lower_bound_on_negative_mutual_information}
	Let $X,Y$ be random variables and let
	\[
	\II{X;Y} = \sum_{x,y} p_{XY}(x, y) \log \frac{p_{XY}(x, y)}{p_{X}(x)p_{Y}(y)} = 
	\]
	\[
	=\sum_{S^+} f(x,y) + \sum_{S^-} f(x,y)
	\]
	be their mutual information, where
	\[
	f(x,y) = p_{XY}(x, y) \log \frac{p_{XY}(x, y)}{p_{X}(x)p_{Y}(y)}
	\]
	\[
	S^+ = \{(x,y): f(x,y) \geq 0\};\	S^- = \{(x,y): f(x,y) < 0\}
	\]
	Then 
	\[
	\sum_{S^-} f(x,y) \geq -1
	\]
\end{lemma}

\begin{proof}[Lemma \ref{lower_bound_on_negative_mutual_information}].
\[
\sum_{S^-} f(x,y) = \sum_{S^-} p_{XY}(x, y) \log \frac{p_{XY}(x, y)}{p_{X}(x)p_{Y}(y)} \geq \Bigg(\sum_{S^-} p_{XY}(x, y)\Bigg) \log \frac{\sum_{S^-} p_{XY}(x, y)}{\sum_{S^-} p_{X}(x)p_{Y}(y)} =
\]
\[
= p_{XY}(S^-) \log \frac{p_{XY}(S^-)}{\sum_{S^-} p_{X}(x)p_{Y}(y)} = p_{XY}(S^-) \Bigg( \log p_{XY}(S^-)\ -\ \log \sum_{S^-} p_{X}(x)p_{Y}(y)\Bigg) \geq 
\]
\[
\geq p_{XY}(S^-) \log p_{XY}(S^-) \geq \min_{x\in[0,1]} x \log x = -\frac{1}{\mathrm{e}} \geq -1
\]
Where the the first inequality is the log-sum inequality, and the second inequality holds because $\sum_{S^-} p_{X}(x)p_{Y}(y) \leq 1$.
\end{proof}

\begin{theorem}[Minimax, \citealt{neumann1928theorie, neumann1944theory}]\label{von-neumann-minimax}
	Let $X\subseteq \mb{R}^n$, $Y \subseteq \mb{R}^m$ be compact convex sets. If $f:X\times Y \rightarrow \mb{R}$ is a continuous function that is convex-concave, i.e., 
	\begin{itemize}
		\item{$f(\cdot, y):X\rightarrow\mb{R}$ is convex for fixed $y\in Y$, and }
		\item{$f(x, \cdot):Y\rightarrow\mb{R}$ is concave for fixed $x\in X$}
	\end{itemize}
	then
	\[
	\min_{x \in X} \: \max_{y \in Y} f(x,y) = \max_{y \in Y} \: \min_{x \in X} f(x,y)
	\]
\end{theorem}

\begin{theorem}[Minimax, \citealt{sion1958general}]\label{sion-minimax}
	Let $X, Y$ be convex sets, one of which is compact. If $f:X\times Y \rightarrow \mb{R}$ is quasi-convex-concave, i.e.,
	\begin{itemize}
		\item{$f(\cdot, y):X\rightarrow\mb{R}$ is quasi-convex for fixed $y\in Y$, and}
		\item{$f(x, \cdot):Y\rightarrow\mb{R}$ is quasi-concave for fixed $x\in X$}
	\end{itemize}
	and $f$ is upper-semi-continuous--lower-semi-continuous, i.e.,
	\begin{itemize}
		\item{$f(\cdot, y):X\rightarrow\mb{R}$ is upper-semi-continuous for fixed $y\in Y$, and}
		\item{$f(x, \cdot):Y\rightarrow\mb{R}$ is lower-semi-continuous for fixed $x\in X$}
	\end{itemize}
	then
	\[
	\sup_{x \in X} \: \inf_{y \in Y} f(x,y) = \sup_{y \in Y} \: \inf_{x \in X} f(x,y)
	\]
\end{theorem}

\begin{lemma}[Theorem 2.7.4 in \citealt{cover2012elements}]\label{cover06-concave}
	Let $(X,Y) \sim p(x, y) = p(x)p(y|x)$. The mutual information $\II{X;Y}$ is a concave function of $p(x)$ for fixed $p(y|x)$ and a convex
	function of $p(y|x)$ for fixed $p(x)$.
\end{lemma}

\end{document}